\documentclass[11pt,draftcls,onecolumn]{IEEEtran}
\usepackage{cite}
\usepackage{graphicx}
\usepackage{algorithm}
\usepackage{algorithmic}
\usepackage{setspace}
\usepackage{subfig}
\usepackage{epsfig}
\usepackage{url}
\usepackage{amsfonts,amssymb,amsmath,bm,paralist,theorem,cite,ifthen,color,amsmath}
\usepackage{array}
\usepackage{calc}
\usepackage{longtable}
\usepackage{mathrsfs}
\usepackage{epsfig}

\newtheorem{thm}{Theorem}

\newtheorem{assmptn}{Assumption}

\theorembodyfont{\rmfamily}

\newtheorem{rmk}{Remark}

\newenvironment{proof}[1][Proof]{\begin{trivlist}
\item[\hskip \labelsep {\bfseries #1}]}{\end{trivlist}}

\newcommand{\st}{{\rm s.t.}}

\newcommand{\bI}{\mathbf{I}}

\newcommand{\bX}{\mathbf{X}}
\newcommand{\bR}{\mathbf{R}}
\newcommand{\bY}{\mathbf{Y}}
\newcommand{\bx}{\mathbf{x}}
\newcommand{\bS}{\mathbf{S}}

\newcommand{\bw}{\mathbf{w}}
\def\bA{{\mathbf A}}
\def\ba{{\mathbf a}}
\def\bY{{\mathbf Y}}
\def\by{{\mathbf y}}
\def\fd{{\rm d}}

\begin{document}
\title{Computational Intractability  of Dictionary Learning for Sparse Representation}
%

\author{Meisam Razaviyayn$^\dag$\thanks{$^\dag$\;{\color{black}Department of Electrical and Computer Engineering,} University of Minnesota, Minneapolis, 55455, USA. (e-mail:
 \texttt{\{meisam,tsen0058,luozq\}}@umn.edu)}, Hung-Wei Tseng$^\dag$ and Zhi-Quan Luo$^\dag$
\thanks{$^*$ This research is supported in part by the National Science Foundation, grant number DMS-1015346.}}
\maketitle
\begin{abstract}
In this paper we consider the  dictionary learning problem for sparse representation. We first show that this problem is NP-hard by polynomial time reduction of the densest cut problem. Then, using successive convex approximation strategies, we  propose  efficient dictionary learning schemes to solve several practical formulations of this problem to stationary points. Unlike many existing algorithms in the literature, such as K-SVD, our proposed dictionary learning scheme is theoretically guaranteed to converge to the set of stationary points under certain mild assumptions. For the image denoising application, the performance and the efficiency of the proposed dictionary learning scheme are comparable to that of K-SVD algorithm in simulation.
\end{abstract}

\section{Introduction}
\label{sec:intro}
The idea of representing a signal with few samples/observations dates back to the classical result of Kotelnikon, Nyquist, Shannon, and Whittaker  \cite{kotel1933carrying, nyquist1928certain, shannon1949communication, whittaker1915functions,prony1995essai}. This idea has evolved over time, and culminated  to the {\it compressive sensing} concept in recent years \cite{donoho2006compressed, candes2006robust}. The {\it compressive sensing} or {\it sparse recovery} approach relies on the observation that many practical signals can be sparsely approximated in a suitable over-complete basis (i.e., a dictionary). In other words, the signal can be approximately written as a linear combination of only a few components (or {\it atoms}) of the dictionary. This observation is a key to many lossy compression methods such as JPEG and MP3.

Theoretically, the exact sparse recovery is possible with high probability under certain conditions. More precisely, it is demonstrated that if the linear measurement matrix satisfies some conditions such as null space property (NSP) or restricted isometry property (RIP), then the exact recovery is possible  \cite{donoho2006compressed, candes2006robust}. These conditions are satisfied with high probability for different matrices such as Gaussian random matrices, Bernoulli random matrices, and partial random Fourier matrices.

In addition to the theoretical advances, compressive sensing has shown great potential in various applications. For example, in the nuclear magnetic resonance (NMR) imaging application, compressive sensing can help reduce the radiation time \cite{ lustig2006rapid, lustig2005faster}. Moreover, the compressive sensing technique has been successfully applied to many other practical scenarios including sub-Nyquist sampling \cite{gedalyahu2010time, mishali2009blind}, compressive imaging \cite{duarte2008single, marcia2009compressive}, and compressive sensor networks \cite{duarte2005distributed, davenport2010joint}, to name just a few.

In some of the aforementioned applications, the sensing matrix and dictionary are pre-defined using application domain knowledge. However, in most applications, the dictionary is not known a-priori and must be learned using a set of training signals. It has been observed that learning a good dictionary can substantially improve the compressive sensing performance, see \cite{lee2006efficient, kreutz2003dictionary, elad2006image, lewicki2000learning, mairal2008supervised, rubinstein2010dictionaries, aharon2005k}. In these applications, dictionary learning is the most crucial step affecting the performance of the compressive sensing approach.

To determine a high quality dictionary, various learning algorithms have been proposed; see, e.g., \cite{aharon2005k,lee2006efficient,engan1999method,olshausen1997sparse}. These algorithms are typically composed of two major steps: 1) finding an approximate sparse representation of the training signals 2) updating the dictionary using the sparse representation.

In this paper, we consider the dictionary learning problem for sparse representation. We first establish the NP-hardness of this problem. Then we consider different formulations of the dictionary learning problem and propose several efficient algorithms to solve this problem. In contrast to the existing dictionary training algorithms  \cite{engan1999method, aharon2005k, lee2006efficient}, our methods neither solve Lasso-type subproblems nor find the active support of the sparse representation vector at each step; instead, they require only simple inexact updates in closed form. Furthermore, unlike most of the existing methods in the literature, e.g., \cite{lee2006efficient,aharon2005k}, the iterates generated by the proposed dictionary learning algorithms are theoretically guaranteed to converge to the set of stationary points under certain mild assumptions.


\section{Problem Statement}
\label{sec:format}

Given a set of training signals
\small $\small\mathrm{Y} = \{\by_i \in \mathbb{R}^n\mid i=1,2,\ldots,N\},$
\normalsize our task is to find a dictionary
\small $\small\mathrm{A} = \{\ba_i\in \mathbb{R}^n\mid  i=1,2,\ldots,k\}$\normalsize
that can sparsely represent the training signals in the set $\mathrm{Y}$. Let $\bx_i \in \mathbb{R}^k, \;i= 1,\ldots,N$, denote the coefficients of sparse representation of the signal $\by_i$, i.e., $\by_i = \sum_{j=1}^k \ba_j x_{ij}$,  where $x_{ij}$ is the $j$-th component of signal $\bx_i$. By concatenating all the training signals, the dictionary elements, and the coefficients, we can define the matrices $\bY \triangleq [\by_1,\ldots,\by_N]$, $\bA  \triangleq [\ba_1,\ldots,\ba_k]$, and $\bX = [\bx_1,\ldots,\bx_N]$. Having these definitions in our hands, the dictionary learning problem for sparse representation can be stated as
\small
\begin{equation}
\label{eq:OriginalDist}
\small
\min_{\bA,\bX}\;\; {\rm d} (\bY,\bA,\bX) \quad \st\; \bA \in \mathcal{A}, \; \bX \in \mathcal{X},
\normalsize
\end{equation}
\normalsize
where $\mathcal{A}$ and $\mathcal{X}$ are two constraint sets. The function $\fd(\cdot,\cdot,\cdot)$ measures our model goodness of fit.
In the next section, we analyze the computational complexity of one of the most popular forms of problem \eqref{eq:OriginalDist}.

\section{Complexity Analysis}
\label{sec:pagestyle}
Consider a special case of problem \eqref{eq:OriginalDist} by choosing the distance function to be the Frobenius norm and imposing sparsity by considering the constraint set $\mathcal{X} = \{\bX \in \mathbb{R}^{k\times N}\;\big|\; \|\bx_i\|_0 \leq s\}$. Then the optimization problem~\eqref{eq:OriginalDist} can be re-written as
\small
\begin{equation}
\small
\label{eq:OriginalFrob}
\min_{\bA,\bX}\;\; \|\bY-\bA\bX\|_F^2, \ \ \st\; \|\bx_i\|_0 \leq s, \;\forall\; i=1,\ldots,N.
\normalsize
\end{equation}
\normalsize
This formulation is very popular and is considered in different studies; see, e.g., \cite{jiang2012submodular,aharon2005k}. The following theorem characterizes the computational complexity of \eqref{eq:OriginalFrob} by showing its NP-hardness. In particular, we show that even for the simple case of $s=1$ and $k=2$, problem \eqref{eq:OriginalFrob} is NP-hard. To state our result, let us define the following concept: let $(\bA^*,\bX^*)$ be a solution of \eqref{eq:OriginalFrob}. For $\epsilon>0$, we say a point $(\tilde\bA,\tilde\bX)$ is an {\it$\epsilon$-optimal solution} of \eqref{eq:OriginalFrob} if $\|\bY-\tilde\bA \tilde\bX\|_F^2 \leq \|\bY-\bA^* \bX^*\|_F^2 + \epsilon$.
\begin{thm}
\label{thm:NPhard}
Assume $s=1$ and $k=2$. Then finding an $\epsilon$-optimal algorithm for solving \eqref{eq:OriginalFrob} is NP-hard. More precisely, there is no polynomial time algorithm in $N,n, \lceil{\frac{1}{\epsilon}}\rceil$ that can solve \eqref{eq:OriginalFrob} to $\epsilon$-optimality, unless $P = NP$.
\end{thm}
\begin{proof}
The proof is based on the polynomial time reduction of the densest cut problem. The densest cut problem can be stated as follows:\\
{\bf \small Densest Cut Problem:} \normalsize Given a graph $\mathcal{G} = (V,E)$, the goal is to maximize the ratio $\frac{|E(P,Q)|}{|P|\; \cdot|Q|}$ over all the bipartitions $(P,Q)$ of the vertices of the graph $\mathcal{G}$. Here $E(P,Q)$ denotes the set of edges between the two partitions and the operator $|\cdot|$ returns the cardinality of a set.\\
Given an undirected graph $\mathcal{G}$, we put an arbitrary directions on it and we define $\bY'$ to be the incidence transpose matrix of the directed graph. In other words, $\bY' \in \mathbb{R}^{|E| \times |V|}$ with
\begin{itemize}
\item $\bY'_{ij} = 1$ if  edge $i$ leaves vertex $j$
\item $\bY'_{ij} = -1$ if edge $i$ enters vertex $j$
\item $\bY'_{ij} = 0$ otherwise
\end{itemize}
Now let us consider the following optimization problem:
\begin{equation}
\label{eq:DCP}
\min_{\bA',\bX}\;\; \|\bY'-\bA'\bX'\|_F^2\quad \st\;  \|\bx'_i\|_0 \leq s, \; \mathbf{1}^T \bx'_i =1,\;\forall i
\end{equation}
with $s=1$ and $k=2$. \\
{\bf \small Claim 1:} Problem \eqref{eq:DCP} is equivalent to the densest cut problem over the graph~$\mathcal{G}$ \cite{aloise2009np}.\\
{\bf \small Claim 2:} Consider two different feasible points $\bX'_1$ and $\bX'_2$ in problem \eqref{eq:DCP}. Let $\bA'_1$ (resp. $\bA'_2$) be the optimal solution of \eqref{eq:DCP} after fixing the variable $\bX'$ to $\bX'_1$ (resp. $\bX'_2$). Let us further assume that $\|\bY' - \bA'_1\bX_1\| \neq \|\bY' - \bA'_2\bX_2\|$. Then, $| \;\|\bY' - \bA'_1\bX_1\| - \|\bY' - \bA'_2\bX_2\|\;| \geq \frac{16}{N^3}$.\\
The proof of claims 1 and 2 are relegated to the appendix section. Clearly, problem \eqref{eq:DCP} is different from \eqref{eq:OriginalFrob}; however the only difference is in the existence of the extra linear constraint in \eqref{eq:DCP}. To relate these two problems, let us define the following problem:
\begin{equation}
\label{eq:OriginalNPhardness}
\min_{\bA,\bX}\;\; \|\bY-\bA\bX\|_F^2 \quad \st\;   \|\bx_i\|_0 \leq s, \;\forall i.
\end{equation}
where $\bX$ is of the same dimension as $\bX'$, but the matrices $\bY$ and $\bA$ have one more row than $\bY'$ and $\bA'$. Here the matrices $\bY$ and $\bA$ have the same number of columns as $\bY'$ and $\bA'$, respectively.   By giving a special form to the matrix $\bY$, we will relate the optimization problem \eqref{eq:OriginalNPhardness} to \eqref{eq:DCP}. More specifically, each column of $\bY$ is defined as follows:
\[
\by_i = \left[
\begin{array}{c}
M\\
\by_i'\\
\end{array}
\right]
\]
with $M = 6 N^7$. Clearly, the optimization problem~\eqref{eq:OriginalNPhardness} is of the form \eqref{eq:OriginalFrob}.  Let $(\bA^*,\bX^*)$ denote the optimizer of \eqref{eq:OriginalNPhardness}. Then it is not hard to see that the first row of the matrix $\bA^*$ should be nonzero and hence by a proper normalization of the matrices $\bA^*$ and $\bX^*$, we can assume that the first row of the matrix $\bA^*$ is $M$, i.e., $a_{11}^* = a_{12}^* = M$. Define $h(\bA,\bX) \triangleq \|\bY' - \bA \bX\|_F^2$. Let $\bw' = (\bA'^*,\bX'^*)$ denote the minimizer of \eqref{eq:DCP}. Similarly, define $\bw \triangleq (\tilde{\bA}^*,\bX^*)$ where $\tilde{\bA}^* \triangleq \bA_{2:n,:}^*$ is the minimizer of \eqref{eq:OriginalNPhardness}, excluding the first row.  Furthermore, define $\bw_+ \triangleq \left(\tilde{\bA}^*,\bX_{+}^*\right)$, where $\bX_{+}^*$ is obtained by replacing the nonzero entries of $\bX^*$ with one. Having these definitions in our hands, the following claim will relate the two optimization problems \eqref{eq:DCP} and \eqref{eq:OriginalNPhardness}.\\
{\bf \small Claim 3:} $h(\bw) \leq h(\bw') \leq h(\bw_+) \leq h(\bw) + \frac{28}{3 N^3}$.\\
The proof of this claim can be found in the appendix section. \\
Now set $\epsilon = \frac{28}{3 N^3}$. If we can solve the optimization problem~\eqref{eq:OriginalNPhardness} to the $\epsilon$-accuracy, then according to Claim 3, we have the optimal value of problem~\eqref{eq:DCP} with accuracy $\epsilon = \frac{28}{3 N^3}$. Noticing that $\frac{16}{N^3}>\frac{28}{3 N^3}$ and using Claim 2, we can further conclude that  the exact optimal solution of \eqref{eq:DCP} is known; which implies that the optimal value of the original densest cut problem is known (according to Claim 1). The NP-hardness of the densest cut problem will complete the proof.
\end{proof}

\begin{rmk}
Note that in the above NP-hardness result, the input size of $\lceil\frac{1}{\epsilon}\rceil$ is  considered instead of $\lceil \log(\frac{1}{\epsilon})\rceil$. This  in fact implies a stronger result that there is no quasi-polynomial time algorithm for solving \eqref{eq:OriginalFrob}; unless P=NP.
\end{rmk}
It is worth noting that the above NP-hardness result is different from (and is not a consequence of) the compressive sensing NP-hardness result in \cite{natarajan1995sparse}. In fact, for a fixed sparsity level $s$,  the compressive sensing problem is no longer NP-hard, while the dictionary learning problem considered herein remains NP-hard (see Theorem~\ref{thm:NPhard}).

\section{Algorithms}
\label{sec:typestyle}
\subsection{Optimizing the goodness of fit}
In this section, we assume that the function $\fd(\cdot)$ is composed of a smooth part and a non-smooth part for promoting sparsity, i.e., $\fd(\bY,\bA,\bX) = \fd_1(\bY,\bA,\bX) + \fd_2(\bX)$, where $\fd_1$ is smooth and $\fd_2$ is continuous and possibly non-smooth. Let us further assume that the sets $\mathcal{A} , \mathcal{X}$ are closed and convex.  Our approach to solve \eqref{eq:OriginalDist} is to apply the general block successive upper-bound minimization framework developed in \cite{razaviyayn2013unified}. More specifically, we propose to alternately update the variables $\bA$ and $\bX$. Let $(\bA^r,\bX^r)$ be the point obtained by the algorithm at iteration $r$. Then, we select one of the following methods to update the dictionary variable $\bA$ at iteration $r+1$:
\begin{itemize}
\small
\item[(a)]  $\bA^{r+1} \leftarrow \displaystyle{\arg\min_{\bA \in \mathcal{A}}}\; \fd(\bY,\bA,\bX^r) $
\item[(b)]  $\bA^{r+1} \leftarrow \displaystyle{\arg\min_{\bA \in \mathcal{A}}} \; \langle\nabla_\bA \fd_1(\bY,\bA^r,\bX^r), \bA \rangle + \frac{\tau_a^r}{2} \|\bA-\bA^r\|_F^2 = \mathcal{P}_{\mathcal{A}}\left(\bA^r - \frac{1}{\tau_a^r}\nabla_\bA \fd_1(\bY,\bA^r,\bX^r)\right)$
\normalsize
\end{itemize}
and we update the variable $\bX$ by
\begin{itemize}
\small
\item $\bX^{r+1} \leftarrow \displaystyle{\arg\min_{\bX \in \mathcal{X}}} \; \langle\nabla_\bX \fd_1(\bY,\bA^{r+1},\bX^r), \bX \rangle + \frac{\tau_x^r}{2} \|\bX-\bX^r\|_F^2 + \fd_2(\bX)$.
\normalsize
\end{itemize}
Here the operator $\langle\cdot,\cdot\rangle$ denotes the inner product; the superscript $r$ represents the iteration number; the notation $\mathcal{P}_{\mathcal{A}}(\cdot)$ is the projection operator to the convex set $\mathcal{A}$; and  the constants $\tau_a^r \triangleq \tau_a(\bY,\bA^r,\bX^r)$ and $\tau_x^r \triangleq \tau_x(\bY,\bA^{r+1},\bX^r)$ are chosen such that
\small
\begin{equation}
\nonumber
\small
\begin{split}
\fd_1(\bY,\bA,\bX^r)& \leq \fd_1(\bY,\bA^r,\bX^r) +   \langle\nabla_\bA \fd_1(\bY,\bA^r,\bX^r), \bA-\bA^r \rangle \\
&+ \frac{\tau_a^r}{2} \|\bA-\bA^r\|_F^2,\; \forall\; \bA \in \mathcal{A}
\end{split}
\normalsize
\end{equation}
\normalsize
and
\small
\begin{align}
\small
\fd(\bY,\bA^{r+1},\bX)& \leq \fd_1(\bY,\bA^{r+1},\bX^r) + \fd_2(\bX)+ \frac{\tau_x^r}{2} \|\bX-\bX^r\|_F^2 \nonumber\\
+ \langle\nabla_\bX &\fd_1(\bY,\bA^{r+1},\bX^r), \bX-\bX^r \rangle,\; \forall\; \bX \in \mathcal{X}. \label{eq:Xupperbound}
\normalsize
\end{align}
\normalsize
It should be noted that each step of the algorithm requires solving an optimization problem. For the commonly used objective functions and constraint sets, the solution to these optimization problems is often in closed form. In addition, the update rule (b) is the classical gradient projection step which can be viewed as an approximate version of (a). As we will see later, for some special choices of the function $\fd(\cdot)$ and the set $\mathcal{A}$, using (b) leads to a closed form update rule, while (a) does not. In the sequel, we specialize this framework to different popular choices of the objective functions and the constraint sets.\\

\noindent{\it Case I: Constraining the total dictionary norm}\\
For any $\beta>0$, we consider the following optimization problem
\small
\begin{equation}
\small\label{eq:case1}
\min_{\bA,\bX}\;\; \frac{1}{2}\|\bY-\bA\bX\|_F^2 + \lambda \|\bX\|_1 \quad\st\;   \|\bA\|_F^2\leq \beta,
\normalsize
\end{equation}
\normalsize
where $\lambda$ denotes the regularization parameter. By simple calculations, we can check that all the steps of the proposed algorithm can be done in closed form. More specifically, using the dictionary update rule (a) will lead to Algorithm~\ref{alg:case1}.
\begin{algorithm}
\caption{The proposed algorithm for solving \eqref{eq:case1}}
\label{alg:case1}
\begin{algorithmic}
\small
\STATE initialize $\bA$ randomly such that $\|\bA\|_F^2\leq \beta$
\REPEAT
\STATE $\tau_a \leftarrow \sigma_{\max}^2(\bX)$
\STATE $\bX \leftarrow \bX - \mathcal{S}_{\frac{\lambda}{\tau_a}}(\bX - \frac{1}{\tau_a}\bA^T (\bA\bX-\bY))$
\STATE  $\bA \leftarrow \bY\bX^T (\bX\bX^T+\theta \bI)^{-1}$
\UNTIL some convergence criterion is met
\normalsize
\end{algorithmic}
\end{algorithm}
In this algorithm, $\sigma_{\max}(\cdot)$ denotes the maximum singular value; $\theta\geq 0 $ is the Lagrange multiplier of the constraint $\|\bA\|_F^2 \leq \beta$ which can be found using one dimensional search algorithms such as bisection or Newton. The notation $\mathcal{S}(\cdot)$ denotes the component-wise soft shrinkage operator, i.e., $\mathbf{B} = \mathcal{S}_\gamma(\mathbf{C})$ if
\small
\[
\small
\mathbf{B}_{ij} = \left\{\begin{array}{ll}
\mathbf{C}_{ij} - \gamma & {\rm if} \; \mathbf{C}_{ij} > \gamma\\
0 & {\rm if} \; -\gamma\leq \mathbf{C}_{ij} \leq \gamma\\
\mathbf{C}_{ij} + \gamma & {\rm if} \; \mathbf{C}_{ij} < -\gamma\\
\end{array}\right.
\]
\normalsize
where $\mathbf{B}_{ij}$ and $\mathbf{C}_{ij}$ denote the $(i,j)$-th component of the matrices $\mathbf{B}$ and $\mathbf{C}$, respectively.\\

\noindent{\it Case II: Constraining the norm of each dictionary atom}\\
In many applications, it is of interest to constrain the norm of each dictionary atom, i.e., the dictionary is learned by solving:
\small
\begin{equation}
\small
\label{eq:case2}
\min_{\bA,\bX}\;\; \frac{1}{2}\|\bY-\bA\bX\|_F^2 + \lambda \|\bX\|_1\quad\st\; \|\ba_i\|_F^2\leq \beta_i, \;\forall \;i
\normalsize
\end{equation}
\normalsize
In this case, the dictionary update rule (a) cannot be expressed in closed form;  as an alternative, we can use the update rule (b), which is in closed form, in place of (a). This gives Algorithm~\ref{alg:case2}.
\begin{algorithm}
\caption{The proposed algorithm for solving \eqref{eq:case2} and \eqref{eq:case3}}
\label{alg:case2}
\begin{algorithmic}
\small
\STATE For solving \eqref{eq:case2}: initialize $\bA$ randomly s.t. $\|\ba_i\|_F^2\leq \beta_i,\;\forall\; i$
\STATE For solving \eqref{eq:case3}: initialize $\|\bA\|_F^2\leq \beta$ and $\bA\geq 0$
\REPEAT
\STATE $\tau_x \leftarrow \sigma_{\max}^2(\bA)$
\STATE  For solving \eqref{eq:case2}: $\bX \leftarrow \bX - \mathcal{S}_{\frac{\lambda}{\tau_x}}(\bX - \frac{1}{\tau_x}\bA^T (\bA\bX-\bY))$
\STATE  For solving \eqref{eq:case3}: $\bX \leftarrow \mathcal{P}_{\mathcal{X}}\left(\bX - \frac{1}{\tau_x}\bA^T (\bA\bX-\bY) - \lambda\right)$
\STATE $\tau_a \leftarrow \sigma_{\max}^2(\bX)$
\STATE  $\bA \leftarrow \mathcal{P}_\mathcal{A}  \left(\bA - \frac{1}{\tau_a} (\bA\bX - \bY) \bX^T\right)$
\UNTIL some convergence criterion is met
\normalsize
\end{algorithmic}
\end{algorithm}
In this algorithm, the set $\mathcal{A}$ is defined as $\mathcal{A} \triangleq \{\bA \;\big| \;\|\ba_i\|_F^2\leq \beta_i,\;\forall\; i \}$\\

\noindent{\it Case III: Non-negative dictionary learning with the total norm constraint}\\
Consider the non-negative dictionary learning problem for sparse representation:
\small
\begin{equation}
\label{eq:case3}
\min_{\bA,\bX}\;\; \frac{1}{2}\|\bY-\bA\bX\|_F^2 + \lambda \|\bX\|_1 \quad \st\;\;  \|\bA\|_F^2\leq \beta, \; \bA,\bX \geq 0
\end{equation}
\normalsize
Utilizing the update rule (b) leads to Algorithm~\ref{alg:case2}. Note that in this case, projections to the sets $\mathcal{X} = \{\bX \mid \bX\geq0\}$ and $\mathcal{A} = \{\bA \mid \|\bA\|_F^2\leq \beta,\bA\geq 0\}$ are simple. In particular, to project to the set $\mathcal{A}$, we just need to first project to the set of nonnegative matrices first and then project to the set $\tilde{\mathcal{A}} = \{\bA \mid \|\bA\|_F^2\leq \beta\}$.

It is worth noting that Algorithm~\ref{alg:case2} can also be applied to the case where $\mathcal{A} = \{\bA \mid \bA\geq 0, \; \|\ba_i\|_F^2\leq \beta_i,\; \forall\; i\}$, since the projection to the constraint set still remains simple. \\

\noindent{\it Case IV: Sparse non-negative matrix factorization}\\
In some applications, it is desirable to have a sparse non-negative dictionary; see, e.g., \cite{hoyer2004non,potluru2013block,kim2008sparse}. In such cases, we can formulate the dictionary learning problem as:
\small
\begin{equation}
\label{eq:case4}
\min_{\bA,\bX}\; \frac{1}{2}\|\bY-\bA\bX\|_F^2 + \lambda \|\bX\|_1 \;\st\;   \|\ba_i\|_1\leq \theta,\;\forall\; i, \; \bA,\bX \geq 0
\end{equation}
\normalsize
It can be checked that we can again use the essentially same steps of the algorithm in case III to solve \eqref{eq:case4}. The only required modification is in the projection step since the projection should be onto the set $\mathcal{A} = \{\bA\mid\bA\geq 0, \|\ba_i\|_1\leq \theta,\;\forall\; i\}$. This step can be performed in a column-wise manner by updating each column $\ba_i$ to $[\ba_i - \rho_i \mathbf{1}]_+$, where $[\cdot]_+$ denotes the projection to the set of nonnegative matrices and $\rho_i\in \mathbb{R}^+$ is a constant that can be determined via one dimensional bisection.  The resulting algorithm is very similar (but not identical) to the one in \cite{hoyer2004non}. However, unlike the algorithm in \cite{hoyer2004non}, all of our proposed algorithms are theoretically guaranteed to converge, as shown  in Theorem~\ref{thm:convergenceBSUM}.
\begin{thm}
\label{thm:convergenceBSUM}
The iterates generated by the algorithms in cases I-IV converge to the set of stationary points of the corresponding optimization problems.
\end{thm}
\noindent{\it Proof:} Each of the proposed algorithms in cases I-IV  is a special case of the block successive upper-bound minimization approach \cite{razaviyayn2013unified}. Therefore,  \cite[Theorem 2]{razaviyayn2013unified} guarantees the convergence of the proposed methods.

\subsection{Constraining the goodness of fit}
In some practical applications, the goodness of fit level may be known \emph{a-priori}. In these cases, we may be interested in finding the sparsest representation of the data for a given goodness of fit level. In particular, for a given $\alpha>0$, we consider
\begin{equation}
\label{eq:GoFConstrained}
\min_{\bA,\bX}\;\; \|\bX\|_1 \quad
\st\; \fd (\bY,\bA,\bX) \leq \alpha, \;\bA \in \mathcal{A}, \; \bX \in \mathcal{X}.
\end{equation}
For example, when the noise level is known, the goodness of fit function can be set as $\fd(\bY,\bA,\bX) = \|\bY-\bA\bX\|_F^2$. We propose an efficient method (Algorithm~\ref{alg:GoFConstrained}) to solve~\eqref{eq:GoFConstrained}, where the constant $\tau_x$ is chosen according to criterion in \eqref{eq:Xupperbound}.\\
\begin{algorithm}
\caption{The proposed algorithm for solving \eqref{eq:GoFConstrained}}
\label{alg:GoFConstrained}
\begin{algorithmic}
\small
\STATE initialize $\bA$ randomly s.t. $\bA \in \mathcal{A}$ and find a feasible $\bX$
\REPEAT
\STATE $\bar\bX \leftarrow \bX$
\STATE $\bX \leftarrow \arg\min_{\bX \in \mathcal{X}} \; \|\bX\|_1\;\; \st\; \fd_1(\bY,\bA,\bar\bX) + \langle \nabla_\bX \fd_1(\bY,\bA,\bar\bX) , \bX - \bar\bX\rangle + \frac{\tau_x}{2} \|\bX-\bar\bX\|_F^2 + \fd_2(\bX) \leq \alpha$
\STATE  $\bA \leftarrow \arg\min_{\bA \in \mathcal{A}} \;\fd(\bY,\bA,\bX) $
\UNTIL some convergence criterion is met
\normalsize
\end{algorithmic}
\end{algorithm}

It is clear that Algorithm~\ref{alg:GoFConstrained} is not a special case of block coordinate descent method \cite{bertsekas1999nonlinear}
or even the block successive upper-bound minimization method \cite{razaviyayn2013unified}. Nonetheless, the convergence of Algorithm~\ref{alg:GoFConstrained} is guaranteed in light of the following theorem.
\begin{thm}
Assume that $(\bar\bX,\bar\bA)$ is a limit point of the iterates generated by Algorithm~\ref{alg:GoFConstrained}. Furthermore, assume that the subproblem for updating $\bX$ is strictly feasible at $(\bar\bX,\bar\bA)$, i.e., there exists $\tilde\bX \in \mathcal{X}$ such that
\small
$
\small
\fd_1(\bY,\bar\bA,\bar\bX) + \langle \nabla_\bX \fd_1(\bY,\bar\bA,\bar\bX) , \tilde\bX - \bar\bX\rangle + \frac{\tau_x}{2} \|\tilde\bX-\bar\bX\|_F^2 + \fd_2(\tilde\bX) < \alpha.
\normalsize
$
Then $(\bar\bX,\bar\bA)$ is a stationary point of \eqref{eq:GoFConstrained}.
\normalsize
\end{thm}
This theorem is similar to \cite[Property 3]{hong2011sequential}. However, the proof here is different due to the lack of smoothness in the objective function. The proof is omitted due to the space limitation. \\


\section{Numerical Experiments}
\label{sec:illust}

In this section, we apply the proposed sparse dictionary learning method, namely algorithm \ref{alg:case2}, to the image denoising application; and compare its performance with that of the K-SVD algorithm proposed in \cite{elad2006image} (and summarized in Algorithm~\ref{alg:imageDenoise}). As a test case, we use the image of Lena corrupted by additive Gaussian noise with various variances ($\sigma^2$).

In Algorithm~\ref{alg:imageDenoise}, $\bR_{i,j}\bS$ denotes the image patch centered at $(i,j)$ coordinate. In step $2$, dictionary $\bA$ is trained to sparsely represent {\it noisy} image patches by using either K-SVD algorithm or Algorithm \ref{alg:case2}. The term $\bx_{i,j}$ denotes the sparse representation coefficient of the patch~$(i,j)$. In K-SVD, it (approximately) solves $\ell_0$-norm regularized problem \eqref{KSVD_dict} by using orthogonal matching pursuit (OMP) to update $\bX$.
In our approach, we use Algorithm~\ref{alg:case2} with $\mathcal{A} = \{\bA \mid \|\ba_i\| \leq 1, \forall\; i = 1,\cdots, N \}$ to solve the $\ell_1$-penalized dictionary learning formulation \eqref{case2_dict}. We set
$
\mu_{i,j} = c(0.0015 \sigma + 0.2), \;\forall\; i,j,
$ in~\eqref{case2_dict} with $c = \frac{1}{I\times J} \sum_{i,j} \|\bR_{i,j} S\|_2$, and $I \times J$ denotes the total number of image patches. This choice of the parameter $\mu_{ij}$ intuitively  means that we emphasize on sparsity more in the presence of stronger noise. Numerical values $(0.0015,0.2)$ are determined experimentally. The final denoised image $\bS$ is obtained by \eqref{Supdate} and setting $\beta=30/\sigma$, as suggested in \cite{elad2006image}.

\begin{figure}
  \centering
  \includegraphics[width=10.6cm]{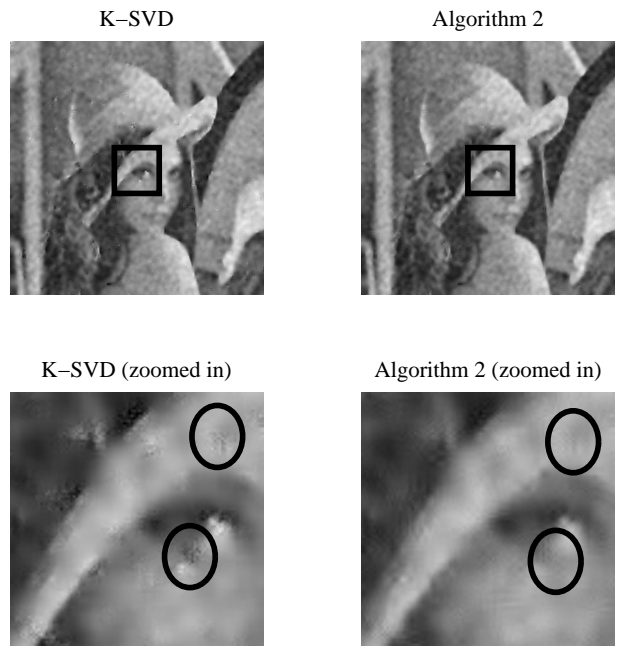}\\
  \vspace{-0.9cm}
  \caption{\footnotesize Sample denoised images ($\sigma = 100$). }\label{bar_imag}
  \vspace{-0.2cm}
\end{figure}
\begin{table}[h]
\centering
\small
\vspace{-0.3cm}
\begin{tabular}{|c|c|c|c|}
  \hline
  {\small$\sigma$/PSNR } & {\small  DCT } & {\small  K-SVD } & {\small  Algorithm \ref{alg:case2}} \\
  {\small 20/22.11} & {\small 32 } & {\small  \textbf{32.38}    } & {\small  30.88} \\
  {\small 60/12.57} & {\small  26.59 } & {\small  \textbf{26.86} } & {\small  26.37} \\
  {\small 100/8.132} & {\small  24.42 } & {\small  24.45 } & {\small  \textbf{24.46}} \\
  {\small 140/5.208} & {\small  22.96 } & {\small  22.93 } & {\small  \textbf{23.11} } \\
  {\small 180/3.025} & {\small  21.73 } & {\small  21.69 } & {\small  \textbf{21.96} }\\
  \hline
\end{tabular}
\caption{\footnotesize Image denoising result comparison on ``Lena" for different noise levels. Values are averaged over $10$ Monte Carlo simulations.}
\label{table_psnr}
\end{table}
\normalsize


\begin{algorithm}[h]
\caption{Image denoising using K-SVD or algorithm \ref{alg:case2}}
\label{alg:imageDenoise}
\begin{algorithmic}[1]
\small
\REQUIRE noisy image $\bY$, noise variance $\sigma^2$
\ENSURE denoised image $\bS$
\STATE Initialization: $\bS = \bY$, $\bA=\text{overcomplete DCT dictionary}$
\STATE Dictionary learning:\\
       K-SVD:
       \small
       \vspace{0cm}
       \begin{align} \label{KSVD_dict}
       \vspace{0cm}
       \min_{\bA,\bX}  \sum_{i,j} \mu_{ij}\|\bx_{i,j}\|_0 + \sum_{i,j} \|\bA \bx_{i,j}  - \bR_{i,j}\bS\|^2
       \end{align}
       \normalsize
       Algorithm \ref{alg:case2}:
       \small
       \begin{align} \label{case2_dict}
       \min_{\bA \in \mathcal{A},\bX}  \sum_{i,j} \mu_{ij}\|\bx_{i,j}\|_1 + \sum_{i,j} \|\bA \bx_{i,j}  - \bR_{i,j}\bS\|^2
       \end{align}
       \normalsize
\STATE $\bS$ update:
\small
    \begin{align} \label{Supdate}
    \bS  = (\beta \bI + \sum_{i,j} \bR_{i,j}^{T}\bR_{i,j} )^{-1} (\beta \bY + \sum_{i,j} \bR_{i,j}^{T}\bA \bx_{i,j})
    \end{align}
    \normalsize
    \normalsize
\end{algorithmic}
\end{algorithm}

The final peak signal-to-noise ratio (PSNR) comparison is summarized in Table~\ref{table_psnr}; and sample images are presented in Figure~\ref{bar_imag}. As can be seen in Table~\ref{table_psnr}, the resulting PSNR values of the proposed algorithm are comparable with the ones obtained by  K-SVD.
However, visually, K-SVD produces more noticeable artifacts (see the circled spot in Figure~\ref{bar_imag}) than our proposed algorithm. The artifacts may be due to the use of OMP in K-SVD which is less robust to noise than the $\ell_1$-regularizer used in  Algorithm~\ref{alg:case2}.
As for the CPU time, the two algorithms perform similarly in the numerical experiments.

\noindent{\bf Acknowledgment:} The authors are grateful to the University of Minnesota Graduate School Doctoral Dissertation Fellowship support during this research.

\appendix[Part I: NP-hardness Proof]

\noindent\textbf{Proof of Claim 1:} This proof is exactly the same as the proof in \cite{aloise2009np}. Here we restate the proof since some parts of the proof is necessary for the proof of Claim 2. Consider a feasible point $(A', X')$ of problem \eqref{eq:DCP}. Clearly, in any column of the matrix $X'$, either the first component is zero, or the second one. This gives us a partition of the columns of the matrix $X'$ (which is equivalent to a partition over the nodes of the graph). Let $P$ (resp. $Q$) be the set of columns of $X'$ for which the first (resp. the second) component is nonzero at the optimality. Define $p \triangleq |P|$ and $q = |Q|$. Then the optimal value of the matrix $\bA = [\ba_1 \ba_2]$ is given by:
\begin{itemize}
\item $a_{j1} = \pm \frac{1}{p}$, $a_{j2} = \mp \frac{1}{q}$ if $j \in E(P,Q)$
\item $a_{j1} = a_{j2} = 0$ if $j \notin E(P,Q)$
\end{itemize}
where $a_{ji}$ is the $j$-th component of column $i$ in matrix $\bA$. Plugging in the optimal value of the matrix $\bA$, the objective function of \eqref{eq:DCP} can be rewritten as:
\begin{align}
\|\bY' - \bA' \bX'\|_F^2 &= \sum_{i\in P} \|\by'_i - \ba'_1\|^2 + \sum_{i\in Q} \|\by'_i - \ba'_2\|^2 \nonumber\\
& = \sum_{j \notin E(P,Q)} 2 + \sum_{j \in E(P,Q)} \left[(1-\frac{1}{p})^2 + \frac{p-1}{p^2} +(1-\frac{1}{q})^2 + \frac{q-1}{q^2}\right] \nonumber\\
& = 2\left(|E| -|E(P,Q)|\right) + |E(P,Q)|(\frac{p-1}{p} + \frac{q-1}{q}) \nonumber\\
& = 2|E| - |E(P,Q)| (\frac{1}{p} + \frac{1}{q})\nonumber\\
& = 2|E| - |V|\frac{|E(P,Q)|}{p\cdot q} = 2 n - N \frac{|E(P,Q)|}{p.q}. \label{eq:ObjDCP}
\end{align}
Hence, clearly, solving $\eqref{eq:DCP}$ is equivalent to solving the densest cut problem on graph $\mathcal{G}$. $\blacksquare$\\

\noindent\textbf{Proof of Claim 2:}
According to the proof of Claim 1, we can write
\begin{align}
\bigg| \|\bY' - \bA'_1 \bX'_1\|_F^2 -\|\bY' - \bA'_2 \bX'_2\|_F^2 \bigg| &= N \bigg|\frac{|E(P_1,Q_1)|}{p_1 q_1} - \frac{|E(P_2,Q_2)|}{p_2 q_2}\bigg| \nonumber\\
& \geq \frac{N}{p_1 (N-p_1) p_2 (N-p_2)} \nonumber\\
& \geq \frac{N}{(N/2)^2} = \frac{16}{N^3}.\quad \;\;\;\blacksquare \nonumber
\end{align}

\noindent\textbf{Proof of Claim 3:}
First of all, notice that the point
\[
\bX = \left[
\begin{array}{cccc}
1 & 1 & \cdots & 1 \\
0 & 0 & \cdots & 0 \\
\end{array}
\right] \quad {\rm and} \quad
\bA =
\left[
\begin{array}{cc}
M & M \\
0 & 0 \\
\vdots & \vdots\\
0 & 0 \\
\end{array}
\right]
\]
is feasible and it should have a higher objective value than the optimal one. Therefore,
\[
\sum_{i=1}^N (M - M(x^*_{1i} + x^*_{2i}))^2 + h(\bw) \leq \|\bY'\|_F^2 = 2|E| \leq 2N^2
\]
which in turn implies that
\begin{align}
\max_i\{|1 - x^*_{1i} - x^*_{2i}|\} \leq \frac{\sqrt{2}N}{M} =\frac{1}{3N^6} \triangleq \delta, \label{eq:deltabound}
\end{align}
since $h(\bw) \geq 0$. Clearly, $\delta<\frac{1}{2}$ and moreover notice that for each $i$ only one of the elements $x^*_{1i}$ and $x^*_{2i}$ is nonzero. Therefore, any nonzero element $x^*_{ij}$ should be larger than $\frac{1}{2}$. On the other hand, due to the way that we construct $\bY'$, we have $|y'_{ij}| \leq 1, \;\forall i,j$. This implies that $|\tilde{a}_{ij}| \leq 2, \;\forall i,j$, leading to
\begin{align}
\|\tilde{\mathbf{a}}_1\|^2 , \|\tilde{\mathbf{a}}_2 \|^2\leq 4N, \label{eq:boundclmA}
\end{align}
where $\tilde{\mathbf{a}}_1$ and $\tilde{\mathbf{a}}_2$ are the first and the second column of matrix $\tilde{\bA}$. Having these simple bounds in our hands, we are now able to bound $h(\bw_+)$:
\begin{align}
h(\bw_+) = &\sum_{i \in P} \|\by'_i - \tilde{\ba}_1\|^2 + \sum_{i \in Q} \|\by'_i - \tilde{\ba}_2\|^2\nonumber\\
& = \sum_{i \in P} \|\by'_i - \tilde{\ba}_1 x_{1i}\|^2 + \sum_{i\in P} \|\ba_1\|^2 (1-x_{1i})^2 + 2\sum_{i\in P}\langle \by'_i - \tilde{\ba}_1 x_{1i} , (x_{1i}-1)\tilde{\ba}_1\rangle \nonumber\\
& + \sum_{i \in Q} \|\by'_i - \tilde{\ba}_2 x_{2i}\|^2 + \sum_{i\in Q} \|\ba_2\|^2 (1-x_{2i})^2 + 2\sum_{i\in Q}\langle \by'_i - \tilde{\ba}_2 x_{2i} , (x_{2i}-1)\tilde{\ba}_2\rangle \nonumber\\
&\leq h(\bw) + \sum_i 4N^2 \delta^2 + 2\sum_{i \in P} (\|\by'_i\| + x_{1i}\|\tilde{\ba}_1\|) \cdot \|\tilde{\ba}_1\| \cdot |1-x_{1i}| \nonumber\\
&+ 2\sum_{i \in Q} (\|\by'_i\| + x_{2i}\|\tilde{\ba}_2\|) \cdot \|\tilde{\ba}_2\| \cdot |1-x_{2i}|  \nonumber\\
& \leq h(\bw) + 4N^3 \delta^2 + 2 \sum_{i \in P} (\|\by'_i \| + 4N) 2N \delta + 2 \sum_{i \in Q}(\|\by'_i \| + 4N) 2N \delta  \nonumber\\
& \leq h(\bw) + 4N^3 \delta^2 + 4N\delta(\sqrt{N} \|\bY'\|_F) + 16 N^3\delta \nonumber\\
& \leq h(\bw) + 4N^3 \delta^2 + 4N\delta(\sqrt{N} \|\bY'\|_F) + 16 N^3\delta \nonumber\\
& \leq h(\bw) + 28 N^3\delta \leq h(\bw) +  \frac{28}{3 N^3}.  \label{eq:clm3lastineq}
\end{align}
Furthermore, since $\bw_+$ is a feasible point for \eqref{eq:DCP} and due to the optimality of $w'$, we have
\begin{align}
h(\bw') \leq h(\bw_+). \label{eq:clm3tmp1}
\end{align}
On the other hand,
\begin{align}
h(\bw) \leq h(\bw'); \label{eq:clm3tmp2}
\end{align}
otherwise, we can add the row $[M\quad  M]$ on top of $\bA'$ and get a lower objective for \eqref{eq:OriginalNPhardness}. Combining \eqref{eq:clm3lastineq}, \eqref{eq:clm3tmp1}, and \eqref{eq:clm3tmp2} will conclude the proof. $\blacksquare$

\appendix[Part II: Successive Convex Approximation]
In this part of the appendix, we analyze the performance of the successive convex approximation method which is used in the development of Algorithm~\ref{alg:GoFConstrained}. To the best of our knowledge, very little is known about the convergence of the successive convex approximation method in the general nonsmooth nonconvex setting. Hence here we state our analysis for the general case. To the best of our knowledge, the previous analysis of this method in \cite[Property 3]{hong2011sequential} is for the smooth case only and a special approximation function; where our analysis covers the nonsmooth case and it appears to be much simpler. To state our result, let us first define the successive convex approximation approach. Consider the following optimization problem:
\begin{equation}
\label{eq:SCAOriginalProb}
\begin{split}
\min_x \quad &h_0(x)\triangleq f_0 (x) + g_0(x)\\
\st \quad &h_i(x) \triangleq f_i(x) + g_i(x)\leq 0, \forall i=1,\ldots,m,
\end{split}
\end{equation}
where the function $f_i(x)$ is smooth (possibly nonconvex) and $g_i$ is convex (possibly nonsmooth), for all $i=0,\ldots,m$. A popular practical approach for solving this problem is the successive convex approximation (also known as majorization minimization) approach where at each iteration of the method, a locally tight approximation of the original optimization problem is solved subject to a tight convex restriction of the constraint sets. More precisely, we consider the successive convex approximation method in Algorithm~\ref{alg:SCA}.

\begin{algorithm}
\caption{Successive Convex Approximation Method for Solving \eqref{eq:SCAOriginalProb}}
\label{alg:SCA}
\begin{algorithmic}
\STATE Find a feasible point $x^0$ in \eqref{eq:SCAOriginalProb}, choose a stepsize $\gamma\in (0,1]$, and set $r=0$
\REPEAT
\STATE Set $r \leftarrow r+1$
\STATE Set $\hat{x}^{r}$ to be a solution of the following optimization problem
\begin{align}
\min_x \quad&\tilde{h}_0(x,x^r) \nonumber\\
\st \quad& \tilde{h}_i(x) \leq 0, \quad\forall i=1,\ldots,m. \nonumber
\end{align}
\STATE Set $x^{r+1} \leftarrow \gamma \hat{x}^r  + (1-\gamma)x^r$
\UNTIL some convergence criterion is met
\end{algorithmic}
\end{algorithm}

The approximation functions in the algorithm need to satisfy the following assumptions:
\begin{assmptn}
\label{assmptn:ApproximationFn}
Assume the approximation functions $\tilde{h}_i(\bullet,\bullet), \; \forall i=0,\ldots,m,$ satisfy the following assumptions:
\begin{itemize}
\item $\tilde{h}_i(x,y)$ is continuous in $(x,y)$
\item $\tilde{h}_i(x,y)$ is convex in $x$
\item $\tilde{h}_i(x,y) = \tilde{f}_i(x,y) + g_i(x),\;\forall x,y$
\item Function value consistency: $\tilde{f}_i (x,x) = f_i(x),\;\forall x$
\item Gradient consistency: $\nabla \tilde{f}_i(\bullet,x) (x) = \nabla f_i(x),\;\forall x$
\item Upper-bound: $\tilde{f}_i(x,y) \geq f_i(x),\;\forall x,y$
\end{itemize}
In other words, we assume that at each iteration, we approximate the original functions with some upper-bounds of them which have the same first order behavior.
\end{assmptn}
In order to state our result, we need to define the following condition:\\

\noindent\textbf{Slater condition for SCA:} Given the constraint approximation functions $\{\tilde{h}(\cdot,\cdot)\}_{i=1}^m$, we say that the Slater condition is satisfied at a given point $\bar{x}$ if there exists a point $x$ in the interior of the restricted constraint sets at the point $\bar{x}$, i.e.,
\[
\tilde{h}_i(x,\bar{x}) < 0, \quad \forall i=1,\ldots,m,
\]
for some $x$. Notice that if the approximate constraints are the same as the original constraints, then this condition will be the same as the well-known Slater condition for strong duality.\\

\begin{thm}
Let $\bar{x}$ be a limit point of the iterates generated by Algorithm~\ref{alg:SCA}. Assume Assumption~\ref{assmptn:ApproximationFn} is satisfied and Slater condition holds at the point $\bar{x}$. Then $\bar{x}$ is a KKT point of \eqref{eq:SCAOriginalProb}.
\end{thm}
\begin{proof}
First of all since the approximate functions are upper-bounds of the original functions, all the iterates are feasible in the algorithm. Moreover, due to the upper-bound and function value consistency assumptions, it is not hard to see that 
\[
h_0(x^{r+1}) \leq \tilde{h}_0(x^{r+1},x^r) \leq \gamma\tilde{h}_0(\hat{x}^{r},x^r) + (1-\gamma)\tilde{h}_0 (x^r,x^r)\leq\tilde{h}_0(x^r,x^r) = h_0(x^r),
\]
where the second inequality is  the result of convexity of $\tilde{h}_0(\cdot,x^r)$. Hence, the objective value is nonincreasing and we must have
\begin{align}
\lim_{r \rightarrow \infty} h_0(x^r) = h_0(\bar{x}), \label{eq:limith0}
\end{align}
and
\begin{align}
\lim_{r \rightarrow \infty} \tilde{h}_0(\hat{x}^r,x^r) = h_0(\bar{x}). \label{eq:limithtilde0}
\end{align}
Let $\{x^{r_j}\}_{j=1}^\infty$ be the subsequence converging to the limit point $\bar{x}$. Consider any fixed point $x'$ satisfying 
\begin{align}
\tilde{h}_i(x',\bar{x}) < 0, \;\;\forall i=1,2,\ldots,m.\;\;\label{eq:strctlyfsble}
\end{align}
Then for $j$ sufficiently large, we must have
\[
\tilde{h}_i(x',x^{r_j}) < 0,\;\;\forall i=1,2,\ldots,m,
\]
i.e., $x'$ is a strictly feasible point at the iteration $r_j$. Therefore, 
\[
\tilde{h}_0(\hat{x}^{r_j},x^{r_j})  \leq \tilde{h}_0(x',x^{r_j}), 
\]
due to the definition of $\hat{x}^{r_j}$. Letting $j  \rightarrow \infty$ and using~\eqref{eq:limithtilde0}, we have
\[
\tilde{h}_0(\bar{x},\bar{x})  \leq \tilde{h}_0(x',\bar{x}).
\]
Notice that this inequality holds for any $x'$ satisfying $\eqref{eq:strctlyfsble}$.
Combining this fact with the convexity of $\tilde{h}_i(\cdot,\bar{x})$ and the Slater condition implies that
\begin{align}
\bar{x} \in \arg \min_x \;\;&\tilde{h}_0 (x,\bar{x}) \nonumber\\
\st \quad &\tilde{h}_i(x,\bar{x}) \leq 0,\;\forall i=1,\ldots,m. \nonumber
\end{align}
Since the Slater condition is satisfied, using the gradient consistency assumption, the KKT condition of the above optimization problem implies that there exist $\lambda_1,\ldots,\lambda_m\geq 0$ such that
\begin{align}
& 0 \in \nabla f_0 (\bar{x}) + \partial g_0(\bar{x}) + \sum_{i=1}^m \lambda_i \left(\nabla f_i(\bar{x}) + \partial g_i (\bar{x})\right) \nonumber\\
& \tilde{f}_i(\bar{x},\bar{x}) + g_i(\bar{x}) \leq 0,\;\forall i=1,\ldots,m, \nonumber\\
& \lambda_i \left(\tilde{f}_i(\bar{x},\bar{x}) + g_i(\bar{x})\right) = 0,\;\forall i=1,\ldots,m. \nonumber
\end{align}
Using the upper-bound  and the objective value consistency assumptions, we have
\begin{align}
& 0 \in \nabla f_0 (\bar{x}) + \partial g_0(\bar{x}) + \sum_{i=1}^m \lambda_i \left(\nabla f_i(\bar{x}) + \partial g_i (\bar{x})\right) \nonumber\\
& f_i(\bar{x}) + g_i(\bar{x}) \leq 0,\;\forall i=1,\ldots,m, \nonumber\\
& \lambda_i \left(f_i(\bar{x}) + g_i(\bar{x})\right) = 0,\;\forall i=1,\ldots,m, \nonumber
\end{align}
which completes the proof.
\end{proof}
It is also worth noting that  in the presence of linear constraints, the Slater condition should be considered for the relative interior of the constraint set instead of the interior.

\bibliographystyle{IEEEbib}
\bibliography{strings,refs}

\end{document}